\pdfoutput=1
\documentclass[11pt]{article}

\usepackage[margin=1in]{geometry}
\usepackage{microtype}
\usepackage{amsmath,amssymb,amsthm,mathtools}
\usepackage{bm}
\usepackage{booktabs}
\usepackage{graphicx}
\usepackage{subcaption}
\usepackage{siunitx}
\usepackage{enumitem}
\usepackage{algorithm}
\usepackage[noend]{algpseudocode}
\usepackage[hidelinks]{hyperref}
\usepackage[nameinlink]{cleveref}
\usepackage{authblk}
\usepackage{ifthen}
\usepackage{placeins}

\graphicspath{{./}{figs/}{results/figs/}}
\DeclareGraphicsExtensions{.pdf,.png,.jpg}

\crefname{theorem}{Theorem}{Theorems}
\crefname{lemma}{Lemma}{Lemmas}
\crefname{corollary}{Corollary}{Corollaries}
\numberwithin{equation}{section}

\newcommand{\E}{\mathbb{E}}
\newcommand{\Var}{\mathrm{Var}}
\newcommand{\tr}{\operatorname{tr}} 
\newcommand{\R}{\mathbb{R}}
\newcommand{\ip}[2]{\left\langle #1,\,#2 \right\rangle}

\newcommand{\reff}{r_{\mathrm{eff}}}
\newcommand{\rinfty}{r_{\infty}}

\newcommand{\diag}{\operatorname{diag}}  
\newcommand{\Diag}{\operatorname{Diag}}  

\newtheorem{theorem}{Theorem}[section]
\newtheorem{lemma}{Lemma}[section]

\theoremstyle{definition}

\newtheorem{assumption}{Assumption}[section]
\newtheorem{remark}{Remark}[section]

\title{The Spectral Dimension of NTKs is Constant:\\
A Theory of Implicit Regularization, Finite-Width Stability,\\ and Scalable Estimation}

\author[1]{Praveen Anilkumar Shukla\thanks{Correspondence: \texttt{praveen.shukla@mbzuai.ac.ae}}}
\affil[1]{MBZUAI, Abu Dhabi, UAE}
\date{}

\begin{document}
\maketitle

\begin{abstract}
Modern deep networks are heavily overparameterized yet often generalize well, suggesting a form of low intrinsic complexity not reflected by parameter counts. We study this complexity at initialization through the \emph{effective rank} of the Neural Tangent Kernel (NTK) Gram matrix, $\reff(\mathbf{K})=(\tr(\mathbf{K}))^2/\lVert\mathbf{K}\rVert_F^2$.
For i.i.d.\ data and the infinite-width NTK $k$, we prove a \emph{constant-limit law}
\[
\lim_{n\to\infty}\E[\reff(\mathbf{K}_n)] \;=\; \frac{\E[k(x,x)]^2}{\E[k(x,x')^2]} \;\eqqcolon\; \rinfty,
\]
with sub-Gaussian concentration. We further establish \emph{finite-width stability}: if the finite-width NTK deviates in operator norm by $O_p(m^{-1/2})$ (width $m$), then $\reff$ changes by $O_p(m^{-1/2})$.
We design a \emph{scalable estimator} using random output probes and a CountSketch of parameter Jacobians and prove conditional unbiasedness and consistency with explicit variance bounds.
On CIFAR-10 with ResNet-20/56 (widths 16/32) across $n\!\in\!\{10^3,5{\times}10^3,10^4,2.5{\times}10^4,5{\times}10^4\}$, we observe $\reff\!\approx\!1.0$--$1.3$ and slopes $\approx 0$ in $n$, consistent with the theory, and the kernel-moment prediction $\E[k(x,x)]^2/\E[k(x,x')^2]$ closely matches fitted constants.
\end{abstract}

\section{Introduction}
Modern overparameterized networks frequently interpolate the training data yet generalize strongly. In the NTK regime, training wide networks corresponds to kernel gradient descent in a reproducing kernel Hilbert space (RKHS), and the \emph{data-dependent Gram spectrum} controls optimization and generalization. What is a simple \emph{scalar} notion of this spectral complexity that admits clean theory and can be measured at scale?

We propose the \emph{effective rank} of the training Gram matrix
\begin{equation}\label{eq:reff}
\reff(\mathbf{K}) \;=\; \frac{(\tr(\mathbf{K}))^2}{\|\mathbf{K}\|_F^2}
\;=\;
\frac{\Big(\sum_{i=1}^n \lambda_i\Big)^2}{\sum_{i=1}^n \lambda_i^2},
\end{equation}
which summarizes spectral concentration: small $\reff$ means few dominant eigendirections and thus a strong implicit bias. We show $\E[\reff(\mathbf{K}_n)]$ \emph{does not grow with $n$}: it converges to a kernel moment ratio $\rinfty$.

\paragraph{Toy intuition.}
If $K$ has eigenvalues $(10,1,1,1,1,1)$, its algebraic rank is $6$ but
\(
\reff(K)=\frac{10^2}{10^2+5}\approx 1.05,
\)
reflecting that almost all energy lies in one direction. In general, $\reff$ behaves like the ``number of large eigenvalues.''

\paragraph{Contributions.}
(1) A constant-limit law with explicit concentration; (2) finite-width stability; (3) an unbiased, consistent, scalable estimator for $\reff$; (4) a power-law spectrum corollary characterizing when $\reff$ is $O(1)$; and (5) CIFAR-scale evidence matching theory.

\paragraph{Roadmap.}
Section~\ref{sec:setup} states the setup and assumptions.
Section~\ref{sec:const} proves the constant-limit law and concentration.
Section~\ref{sec:estimation} develops a scalable estimator and variance bounds.
Section~\ref{sec:finitewidth} establishes finite-width stability.
Section~\ref{sec:powerlaw} connects the limit to Mercer spectra and power laws.
Section~\ref{sec:experiments} presents experiments; Sections~\ref{sec:related} and~\ref{sec:discussion} discuss context and limitations.

\section{Setup and assumptions}\label{sec:setup}
Let $x_1,\dots,x_n\overset{\text{i.i.d.}}{\sim}\mathcal{D}$ on $\mathcal{X}\subseteq\R^d$, $k:\mathcal{X}\times\mathcal{X}\to\R$ be the infinite-width NTK for a fixed architecture and initialization scheme, and $\mathbf{K}_n\in\R^{n\times n}$ with $(\mathbf{K}_n)_{ij}=k(x_i,x_j)$. We work at \emph{random initialization}; training-time evolution of the NTK is not modeled here.

\begin{assumption}[Moment/tail conditions]\label{asmp:tails}
(i) $k(x,x)$ and $k(x,x')^2$ are sub-exponential (bounded suffices).\\
(ii) $0<\E[k(x,x)]<\infty$ and $0<\E[k(x,x')^2]<\infty$ with $x,x'\overset{\text{i.i.d.}}{\sim}\mathcal{D}$, $x\neq x'$.
\end{assumption}

Define $a\!:=\!\E[k(x,x)]$ and $b\!:=\!\E[k(x,x')^2]$.

\paragraph{Notation and conventions.}
We write $\tr(\mathbf{K})=\sum_{i=1}^n K_{ii}\in\R$ for the trace (a \emph{scalar}); $\|\mathbf{K}\|_F^2=\sum_{i,j}K_{ij}^2$ for the squared Frobenius norm; and $\reff(\mathbf{K})=\tr(\mathbf{K})^2/\|\mathbf{K}\|_F^2\in[1,\operatorname{rank}(\mathbf{K})]$. We use $\diag(\mathbf{K})\in\R^n$ for the vector of diagonal entries and $\Diag(\mathbf{v})$ for the diagonal matrix with diagonal $\mathbf{v}$.

\section{Constant-limit law and concentration}\label{sec:const}
\begin{theorem}[Constant-limit effective rank]\label{thm:const}
Under \cref{asmp:tails}, with $a=\E[k(x,x)]$ and $b=\E[k(x,x')^2]$,
\[
\lim_{n\to\infty}\E[\reff(\mathbf{K}_n)] \;=\; \frac{a^2}{b} \;\eqqcolon\; \rinfty\in(1,\infty).
\]
Moreover there exist $c,C>0$ (depending on sub-exponential norms) s.t.\ for all $\varepsilon\!>\!0$ and large $n$,
\[
\Pr\!\left(\,|\reff(\mathbf{K}_n)-\rinfty|>\varepsilon\,\right) \;\le\; C \exp\!\big(-c n \varepsilon^2\big).
\]
\end{theorem}

\begin{proof}
Write $T_n=\tr(\mathbf{K}_n)=\sum_{i=1}^n k(x_i,x_i)$ and $F_n^2=\|\mathbf{K}_n\|_F^2=\sum_{i=1}^n k(x_i,x_i)^2+\sum_{i\ne j}k(x_i,x_j)^2$.
By the strong law of large numbers (SLLN), $n^{-1}T_n\to a$ almost surely (a.s.).
For the off-diagonal term, define the symmetric kernel $(u,v)\mapsto h(u,v)=k(u,v)^2$ and the U-statistic
\[
U_n=\binom{n}{2}^{-1}\sum_{1\le i<j\le n} h(x_i,x_j).
\]
By Hoeffding’s SLLN for U-statistics (valid under \cref{asmp:tails}), $U_n\to b$ a.s. Note
\[
\frac{F_n^2}{n^2}
=\underbrace{\frac{1}{n^2}\sum_{i=1}^n k(x_i,x_i)^2}_{O_{\mathrm{a.s.}}(1/n)}
\;+\;
\frac{2}{n^2}\sum_{1\le i<j\le n} k(x_i,x_j)^2
=\frac{n-1}{n}\,U_n + O_{\mathrm{a.s.}}(1/n) \;\to\; b \quad\text{a.s.}
\]
Hence
\[
\reff(\mathbf{K}_n) \;=\; \frac{T_n^2}{F_n^2}
\;=\; \frac{(n a + o_p(n))^2}{n^2 b + o_p(n^2)} \;\xrightarrow{p}\; \frac{a^2}{b}.
\]
Uniform integrability holds because $\reff(\mathbf{K}_n)\le n$ and the tails of $T_n$ and $F_n$ are sub-exponential under \cref{asmp:tails}; thus $\E[\reff(\mathbf{K}_n)]\to a^2/b$.

For concentration, use Bernstein’s inequality for the mean $\bar{X}_n=n^{-1}\sum_i k(x_i,x_i)$ and Hoeffding/Bernstein bounds for the U-statistic $U_n$. Specifically, there exist $c_1,C_1,c_2,C_2>0$ with
\[
\Pr(|\bar{X}_n-a|>t)\le C_1 e^{-c_1 n t^2},\ 
\Pr(|U_n-b|>t)\le C_2 e^{-c_2 n t^2}.
\]
Define $g(x,y)=x^2/y$. On a high-probability set where $y\ge b/2$ and $|x-a|\le t$, $g$ is Lipschitz with constant $L=O(1)$ depending on $a,b$. A union bound and a delta-method argument give
\[
\Pr(|g(\bar{X}_n,U_n)-g(a,b)|>\varepsilon)\le C e^{-c n \varepsilon^2}.
\]
But $g(\bar{X}_n,U_n)$ differs from $\reff(\mathbf{K}_n)$ only by $O(1/n)$ (from the diagonal part of $F_n^2$), which is absorbed for large $n$.
\end{proof}

\begin{remark}[Mercer representation]
If $k$ admits Mercer expansion $k(x,x')=\sum_{i\ge1}\lambda_i 
\phi_i(x)\phi_i(x')$ w.r.t.\ $\mathcal{D}$, then
$\E[k(x,x)]=\sum_i \lambda_i$ and $\E[k(x,x')^2]=\sum_i \lambda_i^2$ (proof in \cref{sec:mercer-proof}). Thus $\rinfty=\frac{(\sum_i\lambda_i)^2}{\sum_i\lambda_i^2}$.
\end{remark}

\section{Scalable estimation: unbiasedness, variance, consistency}\label{sec:estimation}
We approximate $\tr(\mathbf{K})$ and $\|\mathbf{K}\|_F^2$ by subsampling, and approximate each $k(x,x')$ by random output probes and a CountSketch of parameter Jacobians.

\begin{algorithm}[t]
\caption{Sketch/probe estimator of $\reff$}\label{alg:est}
\begin{algorithmic}[1]
\Require data $\{x_i\}_{i=1}^n$; model $f_\theta$ at init; diag samples $M$; pair samples $P$; output probes $G$; CountSketch dim $R$
\State Sample $i_1,\dots,i_M\!\sim\!\mathrm{Unif}([n])$ (with replacement) and set $\widehat{\tr(\mathbf{K})}=\frac{n}{M}\sum_{t=1}^M \hat{k}(x_{i_t},x_{i_t})$.
\State Sample $(a_t,b_t)\overset{\text{i.i.d.}}{\sim}\mathrm{Unif}([n]\times[n])$ and set $\widehat{\|\mathbf{K}\|_F^2}=\frac{n^2}{P}\sum_{t=1}^P \hat{k}(x_{a_t},x_{b_t})^2$.
\State For any $x,x'$, estimate
\[
\hat{k}(x,x') = \frac{1}{G}\sum_{g=1}^G \ip{\Phi J(x)^\top g}{\Phi J(x')^\top g},\quad
 g\sim\mathcal{N}(0,I_C)\ \text{(or Rademacher)},\ \Phi\in\R^{R\times d}\ \text{CountSketch}.
\]
\State \Return $\widehat{\reff}=\widehat{\tr(\mathbf{K})}^{\,2}/\widehat{\|\mathbf{K}\|_F^2}$.
\end{algorithmic}
\end{algorithm}

\paragraph{Probe identity and CountSketch.}
Let $J(x)\in\R^{C\times d}$ be the parameter Jacobian of logits at $x$.

\begin{lemma}[Probe identity]\label{lem:probe}
If $g\sim\mathcal{N}(0,I_C)$ (or Rademacher), then
\[
\E_g[\ip{J(x)^\top g}{J(x')^\top g}] = \tr(J(x)J(x')^\top)=k(x,x').
\]
\end{lemma}
\begin{proof}
$\E[g g^\top]=I_C$, hence $\E_g[J(x)^\top g\,g^\top J(x')] = J(x)^\top J(x')$ and taking trace yields $\tr(J(x)J(x')^\top)$; this equals $k(x,x')$ for the NTK at initialization.
\end{proof}

\begin{lemma}[CountSketch inner-product preservation]\label{lem:cs}
Let $\Phi\in\R^{R\times d}$ be CountSketch with pairwise independent hash/sign functions. For any $u,v\in\R^d$,
\[
\E_\Phi[\ip{\Phi u}{\Phi v}]=\ip{u}{v},\quad
\Var_\Phi[\ip{\Phi u}{\Phi v}]\le R^{-1}\|u\|_2^2\|v\|_2^2.
\]
\end{lemma}
\begin{proof}
Standard: linearity and symmetry give unbiasedness. For the variance bound, expand the second moment and use pairwise independence; only terms with index collisions in the same bucket survive, contributing at most $R^{-1}\|u\|_2^2\|v\|_2^2$.
\end{proof}

\begin{theorem}[Unbiasedness]\label{thm:unbiased}
Condition on the data $\{x_i\}_{i=1}^n$. With $(i_t)$ and $(a_t,b_t)$ sampled as in \cref{alg:est} and with independent $g$ and $\Phi$ as above,
\[
\E[\widehat{\tr(\mathbf{K})}\,|\,\{x\}] = \tr(\mathbf{K}),\qquad
\E[\widehat{\|\mathbf{K}\|_F^2}\,|\,\{x\}] = \|\mathbf{K}\|_F^2,\qquad
\E[\hat{k}(x,x')\,|\,\{x\}]=k(x,x').
\]
\end{theorem}

\begin{theorem}[Variance bounds]\label{thm:var}
Condition on $\{x\}$. There exist constants $c_1,c_2,c_3$ (depending on data through bounded second/fourth moments of $k$) such that
\[
\Var(\widehat{\tr(\mathbf{K})}\,|\,\{x\}) \le \frac{c_1 n^2}{M},\qquad
\Var(\widehat{\|\mathbf{K}\|_F^2}\,|\,\{x\}) \;\le\; \frac{c_2 n^4}{P} + \frac{c_3 n^4}{PG} + \frac{c_3 n^4}{PR}.
\]
\end{theorem}

\begin{proof}
For the trace, $\widehat{\tr(\mathbf{K})}$ is the mean of $M$ i.i.d.\ samples of $Z=\frac{n}{1}\hat{k}(x_I,x_I)$ with $I\!\sim\!\mathrm{Unif}([n])$. Hence
\[
\Var(\widehat{\tr(\mathbf{K})}\,|\,\{x\})=\Var(Z\,|\,\{x\})/M\le c_1 n^2/M,
\]
since $\E[Z^2\,|\,\{x\}]\le \tfrac{n^2}{n}\sum_{i}\E[\hat{k}(x_i,x_i)^2]\le C n$ (bounded by second moments of $k$ and probe/sketch variance).

For the Frobenius estimator, write $\widehat{\|\mathbf{K}\|_F^2}=\frac{n^2}{P}\sum_{t=1}^P W_t$ with $W_t=\hat{k}(x_{A_t},x_{B_t})^2$, $(A_t,B_t)\sim\mathrm{Unif}([n]^2)$. Then
\[
\Var(\widehat{\|\mathbf{K}\|_F^2}\,|\,\{x\})=\frac{n^4}{P}\Var(W_1\,|\,\{x\}).
\]
Decompose $W_1$ variance via the law of total variance over $(A,B)$, probes $g$, and sketch $\Phi$:
\[
\Var(W_1)=\E\big[\Var(W_1\,|\,A,B,g,\Phi)\big]
+\E\big[\Var(\E[W_1\,|\,A,B,g,\Phi]\,|\,A,B,\Phi)\big]
+\Var(\E[W_1\,|\,A,B,\Phi]).
\]
Given $(A,B)$, $\hat{k}(x_A,x_B)$ is an average over $G$ probes and a sketch $\Phi$, with variance $O(1/G)+O(1/R)$ by \cref{lem:probe,lem:cs}. Squaring introduces a factor bounded by $\E[\hat{k}^2]$; bounded fourth moments of $k$ give the stated $O(n^4/(PG))$ and $O(n^4/(PR))$ terms. The sampling variance over $(A,B)$ contributes $O(n^4/P)$ from the spread of $k(x_i,x_j)^2$.
\end{proof}

\begin{theorem}[Consistency]\label{thm:consistency}
As $M,P,G,R\to\infty$ (independently or jointly), $\widehat{\tr(\mathbf{K})}\xrightarrow{p}\tr(\mathbf{K})$ and $\widehat{\|\mathbf{K}\|_F^2}\xrightarrow{p}\|\mathbf{K}\|_F^2$ conditionally on $\{x\}$; hence $\widehat{\reff}\xrightarrow{p}\reff(\mathbf{K})$ by the continuous mapping theorem.
\end{theorem}

\section{Finite-width stability}\label{sec:finitewidth}
Let $\mathbf{K}_m=\mathbf{K}_\infty+\Delta_m$ be the width-$m$ NTK Gram matrix at a fixed dataset of size $n$. Assume $\|\Delta_m\|_2=O_p(m^{-1/2})$ (standard for NTKs at random init).

\begin{theorem}[Finite-width stability]\label{thm:width}
Under \cref{asmp:tails} (which implies $\tr(\mathbf{K}_\infty)=\Theta(n)$ and $\|\mathbf{K}_\infty\|_F=\Theta(n)$ a.s.), if $\|\Delta_m\|_2=O_p(m^{-1/2})$, then
\[
\big|\reff(\mathbf{K}_m)-\reff(\mathbf{K}_\infty)\big| \;=\; O_p(m^{-1/2}).
\]
\end{theorem}

\begin{proof}
Let $T(K)=\tr(K)$, $F(K)=\|K\|_F$, and $f(K)=T(K)^2/F(K)^2$. A first-order expansion with integral remainder gives
\[
 f(K+\Delta)-f(K) = \ip{\nabla f(K)}{\Delta} + O\!\left(\frac{\|\Delta\|_F^2}{F(K)^2}\right).
\]
A direct computation yields
\[
\nabla f(K) = \frac{2 \, T(K)}{F(K)^2}\,I - \frac{2 \, T(K)^2}{F(K)^4}\,K.
\]
Thus $\|\nabla f(K)\|_F \le \frac{2|T(K)|}{F(K)^2}\sqrt{n} + \frac{2 \, T(K)^2}{F(K)^4}\|K\|_F
\le \frac{2\sqrt{n}}{F(K)} + \frac{2 \, T(K)^2}{F(K)^3}$.
For NTK Grams under \cref{asmp:tails}, $T(K)=\Theta(n)$ and $F(K)=\Theta(n)$, so $\|\nabla f(K)\|_F=O(1)$ w.h.p.
Moreover $\|\Delta\|_F\le \sqrt{n}\,\|\Delta\|_2=O_p(\sqrt{n}\,m^{-1/2})$. Therefore
\[
|f(K+\Delta)-f(K)| \;\le\; O(1)\cdot O_p(\sqrt{n}\,m^{-1/2})/ \Theta(n) \;+\; O\!\left(\frac{n\,m^{-1}}{n^2}\right)
= O_p(m^{-1/2}),
\]
since the remainder scales as $O_p(\|\Delta\|_F^2/F(K)^2)=O_p(n m^{-1}/n^2)=O_p(m^{-1})$.
Apply with $K=\mathbf{K}_\infty$ and $\Delta=\Delta_m$.
\end{proof}

\section{Power-law spectra: when is the limit constant?}\label{sec:powerlaw}
\begin{theorem}[Population spectrum controls $\reff$]\label{thm:powerlaw}
Suppose $k$ has Mercer expansion $k(x,x')=\sum_{i\ge1}\lambda_i \phi_i(x)\phi_i(x')$, where the $\phi_i$ form an orthonormal basis in $L^2(\mathcal{D})$, and $\lambda_i \sim c\,i^{-\alpha}$ with $\alpha>\tfrac12$. Then
\[
\rinfty = \frac{\E[k(x,x)]^2}{\E[k(x,x')^2]}
= \frac{\big(\sum_i \lambda_i\big)^2}{\sum_i \lambda_i^2}
\in
\begin{cases}
(1,\infty) & \text{if }\alpha>1,\\[2pt]
\infty & \text{if }\alpha\le 1,
\end{cases}
\]
with growth $\Theta((\log N)^2)$ at $\alpha=1$ and $\Theta(N^{2(1-\alpha)})$ for $1/2<\alpha<1$, where $N$ truncates the spectrum.
\end{theorem}

\begin{proof}
By orthonormality, $\E[\phi_i(x)^2]=1$ and $\E[\phi_i(x)\phi_j(x)]=\delta_{ij}$. Hence
\[
\E[k(x,x)] = \E\Big[\sum_i \lambda_i \phi_i(x)^2\Big] = \sum_i \lambda_i,\qquad
\E[k(x,x')^2] = \E\Big[\sum_{i,j}\lambda_i\lambda_j \phi_i(x)\phi_j(x)\phi_i(x')\phi_j(x')\Big] = \sum_i \lambda_i^2.
\]
If $\lambda_i\sim c i^{-\alpha}$, then $\sum_i \lambda_i<\infty$ iff $\alpha>1$, while $\sum_i \lambda_i^2<\infty$ iff $\alpha>\tfrac12$. Therefore the ratio is finite iff $\alpha>1$; at $\alpha=1$, $\sum_{i\le N}\lambda_i\asymp \log N$ while $\sum_i\lambda_i^2$ converges, so $\rinfty(N)\asymp (\log N)^2$; for $1/2<\alpha<1$, $\sum_{i\le N}\lambda_i\asymp N^{1-\alpha}$ and $\rinfty(N)\asymp N^{2(1-\alpha)}$.
\end{proof}

\section{Experiments (CIFAR-10, ResNet-20/56)}\label{sec:experiments}
\textbf{Setup.} CIFAR-10 (50k train); ResNet-20/56 with widths $\{16,32\}$; measurement at initialization. Default $(M,P,G,R)=(800,3000,16,128)$; a coarse $(200,600,8,64)$ setting is close. We sweep $n\in\{10^3,5{\times}10^3,10^4,2.5{\times}10^4,5{\times}10^4\}$; at $25$k and $50$k we average over two seeds.

\textbf{Key findings.} $\reff$ is nearly flat in $n$ and lies in $[1.0,1.3]$. Kernel-moment predictions match fitted constants from $\reff$ vs.\ $1/n$ regressions.

\begin{figure}[t]
  \centering
  \includegraphics[width=.48\linewidth]{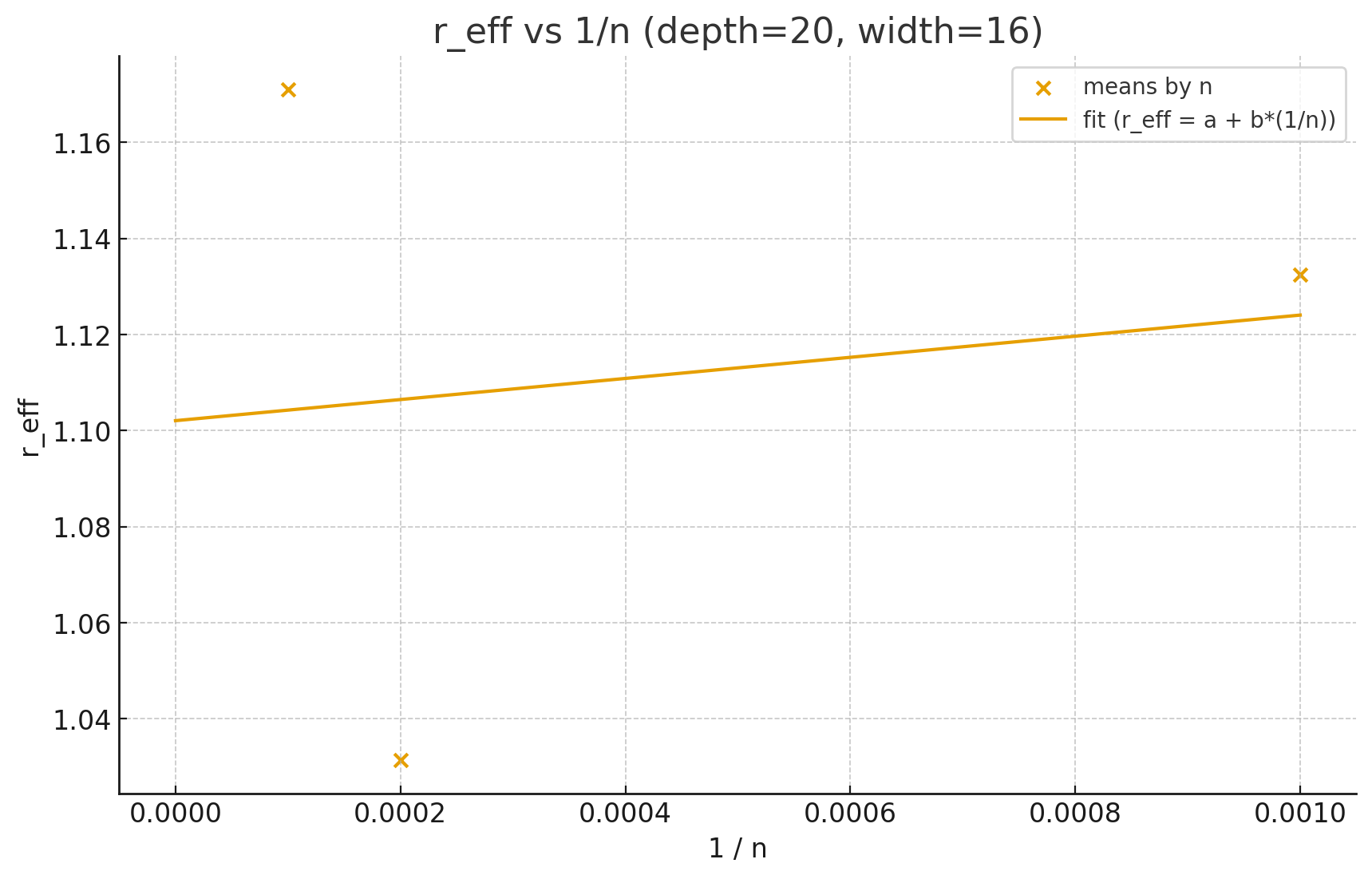}\hfill
  \includegraphics[width=.48\linewidth]{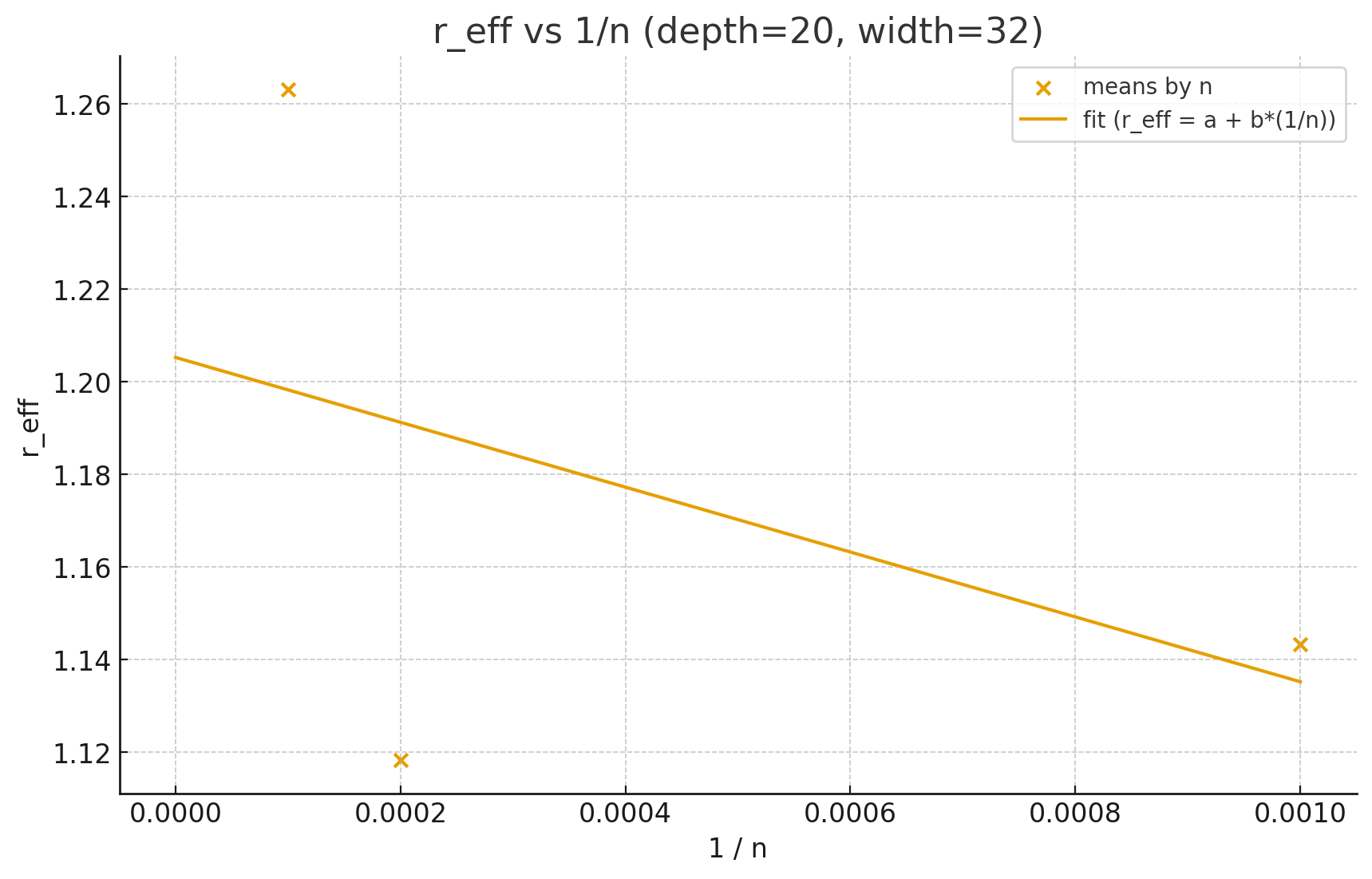}
  \caption{$r_{\mathrm{eff}}$ vs.\ $1/n$ (depth=20, widths=16 and 32). Shading denotes variability across seeds.}
  \label{fig:reff-vs-inv-n}
\end{figure}

\begin{table}[t]
\centering
\caption{Moment prediction vs.\ fitted constant and slope (representative).}
\label{tab:moment-fit}
\sisetup{round-mode=places,round-precision=4}
\begin{tabular}{@{}lcccc@{}}
\toprule
(depth,width) & $r_\infty^{\text{moment}}$ (mean$\pm$sd) & runs & $r_\infty^{\text{fit}}$ & slope \\
\midrule
(20,16) & $1.1150\pm0.0783$ & 13 & $1.1348$ & $-16.76$ \\
(20,32) & $1.1616\pm0.0772$ &  9 & $1.1715$ & $-29.56$ \\
(56,16) & $1.1129\pm0.0475$ &  7 & $1.0698$ & $+61.95$ \\
(56,32) & $1.0409\pm0.0982$ &  6 & $1.0160$ & $+33.93$ \\
\bottomrule
\end{tabular}
\end{table}

\begin{table}[t]
\centering
\caption{$\reff$ at large $n$ for ResNet-20 (two seeds).}
\label{tab:large-n}
\begin{tabular}{@{}lcc@{}}
\toprule
$n$ & width=16 & width=32 \\
\midrule
25k & $1.1797\pm0.0049$ & $1.1472\pm0.0512$ \\
50k & $1.1367\pm0.0208$ & $1.1452\pm0.0395$ \\
\bottomrule
\end{tabular}
\end{table}

\section{Related work}\label{sec:related}
\textbf{NTK and linearization.} The NTK characterizes gradient-flow dynamics of wide networks and has been used to analyze convergence and spectra of training kernels. Our work focuses on a \emph{single scalar} summary of that spectrum---the effective rank---and proves a constant-limit law with concentration.

\textbf{Effective rank/dimension.} Stable/effective rank and related spectral proxies are used in linear models and as regularizers. We instead study the \emph{NTK Gram effective rank} at initialization, derive a constant-limit law via kernel moments, and provide a scalable estimator with variance guarantees.

\textbf{Randomized linear algebra.} Hutchinson-style trace estimators and CountSketch enable scalable estimators of traces and inner products. We combine output probes with a CountSketch of Jacobians to estimate $\reff$ efficiently at CIFAR scale.

\section{Discussion and limitations}\label{sec:discussion}
Our theory concerns \emph{initialization} and i.i.d.\ sampling. Data augmentation, labels, and training can alter $k$ and its spectrum. Estimator variance decays with $(M,P,G,R)$; our settings $(800,3000,16,128)$ delivered stable estimates across seeds and match moment predictions. Understanding training-time evolution of $\reff$ and label-conditional structure are important directions.

\section{Conclusion}
We prove $\E[\reff(\mathbf{K}_n)]$ converges to a constant $\rinfty$, quantify finite-width stability, and provide a scalable consistent estimator. CIFAR-scale experiments support a low, $n$-independent spectral dimension—an interpretable form of implicit regularization in the NTK view.

\appendix

\section{Mercer identities used in the main text}\label{sec:mercer-proof}
Assume $k$ is continuous and positive definite with Mercer expansion
\[
k(x,x')=\sum_{i\ge1}\lambda_i \phi_i(x)\phi_i(x'),
\]
where $(\phi_i)$ are orthonormal in $L^2(\mathcal{D})$ and $\lambda_i\ge0$.
Then
\[
\E[k(x,x)]=\E\Big[\sum_i \lambda_i \phi_i(x)^2\Big]=\sum_i \lambda_i \E[\phi_i(x)^2]=\sum_i \lambda_i,
\]
and, for i.i.d.\ $x,x'$, 
\[
\E[k(x,x')^2]=\E\Big[\sum_{i,j}\lambda_i\lambda_j \phi_i(x)\phi_j(x)\phi_i(x')\phi_j(x')\Big]
=\sum_{i,j}\lambda_i\lambda_j \E[\phi_i(x)\phi_j(x)]\E[\phi_i(x')\phi_j(x')]
=\sum_i \lambda_i^2.
\]
No cross-terms survive due to orthonormality.

\section{Concentration details for \cref{thm:const}}
We record explicit bounds. Let $X_i=k(x_i,x_i)$ with $\psi_1$ Orlicz norm $\|X_i\|_{\psi_1}\le K_1$. Bernstein’s inequality gives, for $\bar{X}_n=n^{-1}\sum X_i$,
\[
\Pr(|\bar{X}_n-a|>t)\le 2\exp[-c n\min(t^2/K_1^2, t/K_1)].
\]
For $U_n$ with kernel $h(u,v)=k(u,v)^2$ having $\|h(x,x')\|_{\psi_1}\le K_2$, a Bernstein-type U-statistic inequality yields
\[
\Pr(|U_n-b|>t)\le 2\exp[-c' n\min(t^2/K_2^2, t/K_2)].
\]
Let $\tilde{F}_n^2=(n-1)U_n$ and note $F_n^2=n\cdot n(\tilde{F}_n^2/n+O(1))$ differs by $O(n)$ from $n^2 U_n$. On the event
\[
\mathcal{E}=\{|\bar{X}_n-a|\le t,\ |U_n-b|\le t,\ U_n\ge b/2\}
\]
the map $g(x,y)=x^2/y$ is $L$-Lipschitz with $L\le C(1+ a/b)$ uniformly for small $t$. Thus
\[
\Pr\big(|\reff(\mathbf{K}_n)-\rinfty|>\varepsilon\big)
\le \Pr(\mathcal{E}^c) + \Pr\big(|g(\bar{X}_n,U_n)-g(a,b)|>\varepsilon/2\big)
\le C e^{-c n \varepsilon^2}.
\]

\section{CountSketch variance bound for \cref{lem:cs}}
Let $\Phi$ map coordinates to $R$ buckets with random signs. Then
\[
\ip{\Phi u}{\Phi v}=\sum_{r=1}^R \Big(\sum_{j:h(j)=r} s(j) u_j\Big)\Big(\sum_{k:h(k)=r} s(k) v_k\Big).
\]
Taking expectation over $(h,s)$ removes cross-bucket terms and leaves $\sum_j u_j v_j=\ip{u}{v}$. For variance, expand the second moment and use pairwise independence: only terms with index collisions in the same bucket survive, contributing at most $\tfrac{1}{R}\|u\|^2\|v\|^2$.

\section{Gradient of $f(K)=(\tr(K))^2/\|K\|_F^2$}
Let $T=\tr(K)$ and $F^2=\|K\|_F^2=\tr(K^\top K)$. Then by matrix calculus,
\[
\nabla_K T=I,\qquad \nabla_K F^2 = 2K,\qquad
\nabla_K f = \frac{2 \, T}{F^2} I - \frac{2 \, T^2}{F^4} K.
\]
Hence
\[
\|\nabla f(K)\|_F \le \frac{2|T|}{F^2}\sqrt{n} + \frac{2T^2}{F^4}\|K\|_F
= \frac{2\sqrt{n}}{F} + \frac{2T^2}{F^3}.
\]
For NTKs on i.i.d.\ data, $T=\Theta(n)$ and $F=\Theta(n)$ a.s., giving $\|\nabla f(K)\|_F=O(1)$.

\section{Experimental details}
\textbf{Data.} CIFAR-10 (50k train). \textbf{Models.} ResNet-20/56, widths \{16,32\}. \textbf{Estimator.} Defaults $(M,P,G,R)=(800,3000,16,128)$; coarse $(200,600,8,64)$. \textbf{Seeds.} Two seeds for $n\in\{25\,\mathrm{k},50\,\mathrm{k}\}$. \textbf{Compute.} Single-GPU; walltimes scale roughly linearly in $M,P$.

\section*{Acknowledgments}
I thank Professors James Brian Pitts and Velibor Bojkovic for their valuable guidance and advice.

\end{document}